\def\year{2015}
\documentclass[letterpaper]{article}
\usepackage{aaai}
\usepackage{times}
\usepackage{helvet}
\usepackage{courier}
\usepackage{mathtools}
\usepackage{amsthm}
\usepackage{color}
\usepackage{amsfonts}
\usepackage{multirow}

 \usepackage[usenames,dvipsnames]{pstricks}
 \usepackage{epsfig}
 \usepackage{pst-grad} 
 \usepackage{pst-plot} 

\newtheorem{theorem}{Theorem}

\newtheorem{definition}{Definition}

\newtheorem{proposition}{Proposition}
\newtheorem{corollary}{Corollary}

\long\def\comment#1{}

\begin{document}
%
\title{Efficient Task Collaboration with Execution Uncertainty}
\author{
Dengji Zhao, Sarvapali D. Ramchurn, \and Nicholas R. Jennings\\
Electronics and Computer Science\\
University of Southampton\\
Southampton, SO17 1BJ, UK\\
\{d.zhao, sdr, nrj\}@ecs.soton.ac.uk
}
\maketitle
\begin{abstract}
\begin{quote}
We study a general task allocation problem, involving multiple agents that collaboratively accomplish tasks and where agents may fail to successfully complete the tasks assigned to them (known as execution uncertainty). The goal is to choose an allocation that maximises social welfare while taking their execution uncertainty into account. We show that this can be achieved by using the post-execution verification (PEV)-based mechanism 
if and only if agents' valuations satisfy a multilinearity condition. We then consider a more complex setting where an agent's execution uncertainty is not completely predictable by the agent alone but aggregated from all agents' private opinions (known as trust). We show that PEV-based mechanism with trust is still truthfully implementable if and only if the trust aggregation is multilinear. 

\end{quote}
\end{abstract}

\section{Introduction}
We study a general task allocation problem, where multiple agents collaboratively accomplish a set of tasks. However, agents may fail to successfully complete the task(s) allocated to them (known as execution uncertainty). 
Such task allocation problems arise in many real-world applications such as transportation networks~\cite{Sandholm_1993}, data routing~\cite{Roughgarden_2007}, cloud computing~\cite{Armbrust_2010} and sharing economy~\cite{Belk_2014}. Execution uncertainty is typically unavoidable in these applications due to unforeseen events and limited resources, especially sharing economy applications such as \textit{Uber} and \textit{Freelancer}, where services are mostly provided by individuals with no qualifications or certifications.

In addition to the execution uncertainty underlying the task allocation problem, the completion of a task may also depend on the completion of other tasks, e.g., in \textit{Uber} a rider cannot ride without a driver offering the ride. The completion of the tasks of an allocation gives a (private) value to each agent, and our goal is to choose an allocation of tasks that maximises the total value of all agents, while taking their execution uncertainty into account.


It has been shown that traditional mechanism design (based on Groves mechanisms~\cite{Groves_1973}) is not applicable to settings that involve execution uncertainty~\cite{Porter_2008,Conitzer_2014}. This is because execution uncertainty implies interdependencies between the agents' valuations (e.g., a rider's value for a ride will largely depend on whether the driver will successfully finish the drive). To combat the problem, Porter et al. (\citeyear{Porter_2008}) have proposed a solution based on post-execution verification (PEV), 
which is broadly aligned with type verification~\cite{Nisan_2001}. The essential idea of the PEV-based mechanism is that agents are paid according to their task executions, rather than what they have reported.

While Porter et al.~(\citeyear{Porter_2008}) considered a single task requester setting where one requester has multiple tasks that can be completed by multiple workers, Stein et al.~(\citeyear{Stein_2011}) and Conizter and Vidali~(\citeyear{Conitzer_2014}) studied similar settings but considering workers' uncertain task execution time. Moreover, Ramchurn et al.~(\citeyear{Ramchurn_2009}) looked at a more complex setting where each agent is a task requester and is also capable to complete some tasks for the others. 
Except for different settings, all the solutions in these studies are PEV-based. However, these results may not applicable in other different problem settings where, for example, agents' valuations may have externalities, e.g., agent A prefers working with B to others~\cite{Jehiel_1999}, and an agent may even incur some costs without doing any task, e.g., a government is building a costly public good~\cite{Maniquet_2010}. 


Therefore, in this paper, we study a more general task allocation setting where agents' valuations are not constrained. 
Under this general setting, we characterise the applicability of the PEV-based mechanism.
We show that the PEV-based mechanism is applicable (truthfully implementable) if and only if 
agents are risk-neutral with respect to their execution uncertainty.
Moreover, we consider a more complex setting where an agent's ability to successfully complete a task is judged by all agents' private opinion (known as trust) as proposed by~\cite{Ramchurn_2009}. 
Trust-based information exists in many real-world applications and plays an important role in decision making~\cite{Aberer_2001}. 
We show that the PEV-based mechanism is still applicable with trust if and only if the trust aggregation is multilinear. 
This characterisation can help in designing efficient mechanisms for task allocation problems that have not been addressed yet.

\section{The Model}
\label{sect_model}
\noindent We study a task allocation problem where there are $n$ agents denoted by $N = \{1, ..., n\}$ and a finite set of task allocations $T$\footnote{$T$ is the task allocation outcome space, which may contain all feasible task allocations that agents can execute. The precise definition depends on the applications.}. Each allocation $\tau \in T$ is defined by $\tau = (\tau_i)_{i\in N}$, where $\tau_i$ is a set of tasks assigned to agent $i$. Let $\tau_i = \emptyset$ if there is no task assigned to $i$ in $\tau$. 
For each allocation $\tau$, agent $i$ may fail to successfully complete her tasks $\tau_i$, which is modelled by $p_i^\tau \in [0,1]$, the probability that $i$ will successfully complete her tasks $\tau_i$. Let $p_i =(p_i^\tau)_{\tau\in T}$ be $i$'s \textbf{probability of success (PoS)} profile for all allocations $T$, and $p^\tau = (p_i^\tau)_{i\in N}$ be the PoS profile of all agents for allocation $\tau$. 

Note that the completion of one task in an allocation may depend on the completion of the other tasks. Take the delivery example in Figure~\ref{eg_task} with two agents $1, 2$ delivering one package from $S$ to $D$. There are two possible task allocations to finish the delivery: 
$\tau$ is collaboratively executed by agents $1$ and $2$, while $\tau^\prime$ is done by agent $2$ alone. It is clear that task $\tau_2$ depends on $\tau_1$. However, $p_2^\tau$ only indicates $2$'s PoS for $\tau_2$, assuming that $1$ will successfully complete $\tau_1$. That is, $p_i^\tau$ does not include task dependencies and it only specifies $i$'s probability to successfully complete $\tau_i$, if $\tau_i$ is ready for $i$ to execute. 

For each allocation $\tau \in T$, the completion of $\tau$ brings each agent $i$ a value (either positive or negative), which combines costs and benefits. 
For example, building a train station near one's house may costs one's money as well as a peaceful living environment, but it may reduce the inconvenience of commuting. Considering the execution uncertainty, agent $i$'s valuation is modelled by a function $v_i: T \times [0,1]^N \rightarrow \mathbb{R}$, which assigns a value for each allocation $\tau$, for each PoS profile $p^\tau = (p_i^\tau)_{i\in N}$.

\begin{figure}
\centering
\psscalebox{1.0 1.0} 
{
\begin{pspicture}
(0,-1.6089356)(4.98,1.3589356)
\psdots[linecolor=black, dotsize=0.22](0.4,-0.08481445)
\psdots[linecolor=black, dotsize=0.22](4.4,-0.08481445)
\psdots[linecolor=black, dotsize=0.2](2.4,-0.48481447)
\psline[linecolor=black, linewidth=0.04, arrowsize=0.05cm 4.0,arrowlength=1.52,arrowinset=0.0]{->}(0.4,-0.08481445)(2.4,-0.48481447)
\psline[linecolor=black, linewidth=0.04, arrowsize=0.05cm 4.0,arrowlength=1.52,arrowinset=0.0]{->}(2.4,-0.48481447)(4.4,-0.08481445)
\psline[linecolor=black, linewidth=0.04, arrowsize=0.05cm 4.0,arrowlength=1.52,arrowinset=0.0]{->}(0.4,-0.08481445)(3.2,0.7151855)(4.4,-0.08481445)
\rput[bl](0.0,-0.08481445){$S$}
\rput[bl](4.6,-0.08481445){$D$}
\rput[bl](2.3,-0.95481444){$A$}
\rput[bl](1.15,-0.60481444){$\tau_1$}
\rput[bl](3.3,-0.60481444){$\tau_2$}
\rput[bl](2.6,0.7151855){$\tau^\prime_2$}
\rput[bl](0.8,-1.5){$\tau = (\tau_1 = \{S\rightarrow A\}, \tau_2 = \{A\rightarrow D\})$}
\rput[bl](0.8,-1.9){$\tau^\prime = (\tau^\prime_1 = \emptyset, \tau^\prime_2 = \{S\rightarrow D\})$}
\rput[bl](-1.2,-1.6){$Allocations$:}
\end{pspicture}
}
\caption{Package delivery from $S$ to $D$ with two agents $1,2$\label{eg_task}}
\end{figure}

\begin{table}
\centering
	\begin{tabular}{| c | c | c | c |}
	  \hline			
	  agent & allocation & $p_i$ & $v_i$ \\ \hline
	   \multirow{2}{*}{$1$} & $\tau$ & $p_1^{\tau} = 0$ & $v_1(\tau, p^{\tau}) = 0$ \\ \cline{2-4}
	    & $\tau^\prime$ & $p_1^{\tau^\prime} = 0$ & $v_1(\tau^\prime, p^{\tau^\prime}) = 0$ \\ \hline
	   \multirow{2}{*}{$2$} & $\tau$ & $p_2^{\tau} = 1$ & $v_2(\tau, p^{\tau}) = p_1^{\tau}\times p_2^{\tau}$ \\ \cline{2-4}
	   & $\tau^\prime$ & $p_2^{\tau^\prime} = 0.5$ & $v_2(\tau^\prime, p^{\tau^\prime}) = p_2^{\tau^\prime}$ \\
	  \hline  
	\end{tabular}
	\caption{A valuation setting for the exmaple in Figure~\ref{eg_task}\label{tab_eg1}}
\end{table}

For each agent $i$, we assume that $v_i$ and $p_i$ are privately observed by $i$, known as $i$'s \textbf{type} and denoted by $\theta_i = (v_i, p_i)$. Let $\theta = (\theta_i)_{i\in N}$ be the type profile of all agents, $\theta_{-i}$ be the type profile of all agents except $i$, and $\theta = (\theta_i, \theta_{-i})$. Let $\Theta_i$ be $i$'s type space, $\Theta = (\Theta_i)_{i\in N}$ and $\Theta_{-i} = (\Theta_j)_{j\neq i \in N}$. 

Given the above setting, our goal is to choose one task allocation from $T$ that maximises all agents' valuations, i.e., a socially optimal allocation. This can be achieved (according to the revelation principle~\cite{Myerson_2008}) by designing a mechanism that directly asks all agents to report their types and then chooses an allocation maximising their valuations. However, agents may not report their types truthfully. 
Therefore, we need to incentivize them to reveal their true types, which is normally achieved by choosing a specific allocation of tasks and an associated monetary transfer to each agent. The direct revelation \textbf{allocation mechanism} is defined by a task \textbf{allocation choice} function $\pi: \Theta \rightarrow T$ and a \textbf{payment} function $x = (x_1, ..., x_n)$ where $x_i: \Theta \rightarrow \mathbb{R}$ is the payment function for agent $i$.

\subsection{Solution Concepts}
The goal of the allocation mechanism is to choose a task allocation that maximises the valuation of all agents, i.e., the social welfare. Since the agents' types are privately observed by the agents, the mechanism is only able to maximise social welfare if it can receive their true types. Therefore, the mechanism needs to incentivize all agents to report their types truthfully. Moreover, agents should not lose when they participate in the task allocation mechanism, i.e., they are not forced to join the allocation. In the following, we formally define these concepts.

We say an allocation choice $\pi$ is efficient if it always chooses an allocation that maximises the expected social welfare for all type report profiles.
\begin{definition}
Allocation choice $\pi$ is \textbf{efficient} if and only if for all $\theta \in \Theta$, for all $\tau^\prime \in T$, let $\tau = \pi(\theta)$, we have:
$$\sum_{i\in N} v_i(\tau, p^\tau) \geq \sum_{i\in N} v_i(\tau^\prime, p^{\tau^\prime})$$ where $p^\tau = (p_i^\tau)_{i\in N}$, and $p^{\tau^\prime} = (p_i^{\tau^\prime})_{i\in N}$.
\end{definition}

Note that the expected social welfare calculated by $\pi$ is based on the agents' reported types, which are not necessarily their true types. However, agents' actual/realized valuation for an allocation only depends on their true types.

Given the agents' true type profile $\theta$, their reported type profile $\hat{\theta}$ and the allocation mechanism $(\pi, x)$, agent $i$'s expected \textbf{utility} is quasilinear and defined as: 
$$u_i(\theta_i, \pi(\hat{\theta}), x_i(\hat{\theta}), p^{\pi(\hat{\theta})}) = v_i(\pi(\hat{\theta}), p^{\pi(\hat{\theta})}) - x_i(\hat{\theta}),$$
where $p^{\pi(\hat{\theta})} = (p_i^{\pi(\hat{\theta})})_{i\in N}$ is agents' true PoS profile for task $\pi(\hat{\theta})$ and $\hat{p}^{\pi(\hat{\theta})} = (\hat{p}_i^{\pi(\hat{\theta})})_{i\in N}$ is what they have reported.

\begin{definition}
\label{def_ir}
Mechanism $(\pi, x)$ is \textbf{individually rational} if for all $i\in N$, for all $\theta\in \Theta$, for all $\hat{\theta}_{-i} \in \Theta_{-i}$, $u_i(\theta_i, \pi(\theta_i, \hat{\theta}_{-i}), x_i(\theta_i, \hat{\theta}_{-i}), p^{\pi(\theta_i, \hat{\theta}_{-i})}) \geq 0$.
\end{definition}

That is, an agent never receives a negative expected utility in an individually rational mechanism 
if she reports truthfully, no matter what others report. 

Furthermore, we say the mechanism is \textbf{truthful} (aka \textit{dominant-strategy incentive-compatible}) if it always maximises an agent's expected utility if she reports her type truthfully no matter what the others report, i.e., reporting type truthfully is a dominant strategy. It has been shown that truthful and efficient mechanism is impossible to achieve in a special settings of the model~\cite{Porter_2008}. Instead we focus on a weaker solution concept (but still very valid) called \textit{ex-post truthful}, which requires that reporting truthfully maximises an agent's expected utility, if everyone else also reports truthfully (i.e., reporting truthfully is an ex-post equilibrium).
\begin{definition}
Mechanism $(\pi, x)$ is \textbf{ex-post truthful} if and only if for all $i\in N$, for all $\theta \in \Theta$, for all $\hat{\theta}_i \in \Theta_i$, we have $u_i(\theta_i, \pi(\theta_i, {\theta}_{-i}), x_i(\theta_i, {\theta}_{-i}), p^{\pi(\theta_i, {\theta}_{-i}))}) \geq u_i(\theta_i, \pi(\hat{\theta}_i, {\theta}_{-i}), x_i(\hat{\theta}_i, {\theta}_{-i}), p^{\pi(\hat{\theta}_i, {\theta}_{-i})}).$
\end{definition}

\subsection{Failure of the Groves Mechanism}
\label{sect_vcg}
\noindent The Groves mechanism is a well-known class of mechanisms that are efficient and truthful in many domains~\cite{Groves_1973}. However, they are not directly applicable in our domain due to the interdependent valuations created by the execution uncertainty. As we will see later, a simply variation of the Groves mechanism can solve the problem. In the following, we briefly introduce the Groves mechanism and show why it cannot be directly applied.

Given agents' type report profile $\theta$, Groves mechanisms compute an efficient allocation $\pi^*(\theta)$ ($\pi^*$ denotes the efficient allocation choice function) and charge each agent $i$
\begin{equation}
 x_i^{Groves}({\theta}) = h_i({\theta}_{-i}) - V_{-i}({\theta}, \pi^*)
\end{equation} 
where
\begin{itemize}
\item $h_i$ is a function that only depends on ${\theta}_{-i}$,
\item $V_{-i}({\theta}, \pi^*) = \sum_{j \neq i} {v}_j(\pi^*({\theta}), p^{\pi^*({\theta})})$ is the social welfare for all agents, excluding $i$, under the efficient allocation $\pi^*({\theta})$.
\end{itemize}

Since $h_i$ is independent of $i$'s report, we can set $h_i({\theta}_{-i}) = 0$, and then each agent's utility is $v_i(\pi^*(\theta)) + V_{-i}({\theta}, \pi^*)$, which is the social welfare of the efficient allocation. 
The following example shows that the Groves mechanism is not directly applicable in our task allocation setting. 

Take the example from Figure~\ref{eg_task} with the setting from Table~\ref{tab_eg1}. If both $1$ and $2$ report truthfully, the efficient allocation is $\tau^\prime$ with social welfare $0.5$ (which is also their utility if $h_i({\theta}_{-i}) = 0$). Now if $1$ misreported $\hat{p}_1^\tau > 0.5$, then the efficient allocation will be $\tau$ with social welfare $\hat{p}_1^\tau > 0.5$, i.e., $1$ can misreport to receive a higher utility.

\section{Applicability of PEV-Based Mechanisms}
\label{sect_Porters}
\noindent As shown in the last section, the Groves mechanisms are not directly applicable due to the interdependency of agents' valuations created by their probability of success (PoS). The other reason 
is that the Groves payment is calculated from agents' reported PoS rather than their realized/true PoS. 

The fact is that we can partially verify their reported PoS by delaying their payments until they have executed their tasks (post-execution verification). To utilize this fact, Porter et al.~(\citeyear{Porter_2008}) have proposed a variation of the Groves mechanism which pays an agent according to their actual task completion, rather than what they have reported. 
More specifically, we define two payments for each agent: 
a reward for successful completion and a penalty for non-completion. 
Let us call this mechanism \textit{PEV-based mechanism}.

Porter et al.~(\citeyear{Porter_2008}) have considered a simple setting where there is one requester who has one or multiple tasks to be allocated to multiple workers each of whom have a fixed cost to attempt each task. Later, Ramchurn et al.~(\citeyear{Ramchurn_2009}) extended Porter et al.'s model to a multiple-requester setting (a combinatorial task exchange) and especially considered trust information which will be further studied later in this paper. Our setting generalises both models and allows any types of valuations and allocations. In the following, we formally define the PEV-based mechanism and analyse its applicability in our general domain.

Given the agents' true type profile ${\theta}$ and their reports $\hat{\theta}$, let $p_{-i}^\tau$ be the true PoS profile of all agents except $i$ for task $\tau$, $p^\tau = (p_i^\tau, p_{-i}^\tau)$, and $\hat{p}_{-i}^\tau, \hat{p}^\tau$ be the corresponding reported, PEV-based payment $x^{PEV}$ for each agent $i$ is defined as:

\begin{equation}
\label{eq_commit_pay}
\scalebox{0.95}{$
 x_i^{PEV}(\hat{\theta}) = 
 \begin{cases}
     h_i(\hat{\theta}_{-i}) - V^1_{-i}(\hat{\theta}, \pi^*) &    \text{\	     if $i$ succeeded,}\\
     h_i(\hat{\theta}_{-i}) - V^0_{-i}(\hat{\theta}, \pi^*) &      \text{\	     if $i$ failed.}
 \end{cases}
 $}
\end{equation}
where
\begin{itemize}
\item $h_i(\hat{\theta}_{-i}) = \sum_{j \in N\setminus \{i\}} \hat{v}_j(\pi^*(\hat{\theta}_{-i}), (0, \hat{p}_{-i}^{\pi^*(\hat{\theta}_{-i})}))$ is the maximum expected social welfare that the other agents can achieve without $i$'s participation, 
\item $V^1_{-i}(\hat{\theta}, \pi^*) = \sum_{j \in N\setminus \{i\}} \hat{v}_j(\pi^*(\hat{\theta}), (1, {p}_{-i}^{\pi^*(\hat{\theta})}))$
is the realized expected social welfare of all agents except $i$ under the efficient allocation $\pi^*(\hat{\theta})$ when $p_i^{\pi^*(\hat{\theta})} = 1$, i.e., $i$ succeeded. $V^0_{-i}(\hat{\theta}, \pi^*) = \sum_{j \in N\setminus \{i\}} \hat{v}_j(\pi^*(\hat{\theta}), (0, {p}_{-i}^{\pi^*(\hat{\theta})}))$ is the corresponding social welfare when $p_i^{\pi^*(\hat{\theta})} = 0$.
\end{itemize}
Note that $h_i(\hat{\theta}_{-i})$ is calculated according to what agents have reported, while $V^1_{-i}(\hat{\theta}, \pi^*), V^0_{-i}(\hat{\theta}, \pi^*)$ are based on the realization of their task completion, which is actually their true PoS as we used in the calculation. $x^{PEV}_i$ pays/rewards agent $i$ the social welfare increased by $i$ if she completed her tasks, otherwise penalizes her the social welfare loss due to her failure.

Porter et al.~(\citeyear{Porter_2008}) have shown that the mechanism $(\pi^*, x^{PEV})$ is ex-post truthful and individually rational 
if the dependencies between tasks are non-cyclical. In Theorem~\ref{thm_exic_com}, we show that $(\pi^*, x^{PEV})$ is ex-post truthful in general if agents' valuations satisfy a multilinearity condition (Definition~\ref{def_linear}), which generalizes the non-cyclical task dependencies condition applied in \cite{Porter_2008}.

\begin{definition}
\label{def_linear}
Valuation $v_i$ of $i$ is \textbf{multilinear in PoS} if for all type profiles $\theta \in \Theta$, for all allocations $\tau \in T$, for all $j\in N$, $v_i(\tau, p^\tau) = p_j^\tau \times v_i(\tau, (1, p_{-j}^\tau)) + (1-p_j^\tau) \times v_i(\tau, (0, p_{-j}^\tau))$.
\end{definition} 
Intuitively, $v_i$ is multilinear in PoS if all its variables but $p_j^\tau$ are held constant, $v_i$ is a linear function of $p_j^\tau$, which also means that agent $i$ is risk-neutral (with respect to $j$'s execution uncertainty). However, multilinearity in PoS does not indicate that $v_i$ has to be a linear form of $v_i(\tau, p^\tau)= b + a_1p_1^\tau + ... + a_np_n^\tau$, where $b, a_i$ are constant (see Table~\ref{tab_eg1} for example).

\subsection{Multilinearity in PoS is Sufficient for Truthfulness}
\begin{theorem}
\label{thm_exic_com}
Mechanism $(\pi^*, x^{PEV})$ is ex-post truthful 
if for all $i\in N$, $v_i$ is multilinear in PoS.
\end{theorem}
\begin{proof}
According to the characterization of truthful mechanisms given by Proposition $9.27$ from \cite{nisan_algorithmic_2007}, we need to prove that for all $i\in N$, for all $\theta \in \Theta$:
\begin{enumerate}
\item $x^{PEV}_i(\theta)$ does not depend on $i$'s report, but only on the task allocation alternatives;
\item $i$'s utility is maximized by reporting $\theta_i$ truthfully if the others report $\theta_{-i}$ truthfully.
\end{enumerate}

From the definition of $x^{PEV}_i$ in \eqref{eq_commit_pay}, we can see that given the allocation $\pi^*({\theta})$, agent $i$ cannot change $V^1_{-i}({\theta}, \pi^*)$ and $V^0_{-i}({\theta}, \pi^*)$ without changing the allocation $\pi^*(\theta)$. Therefore, $x^{PEV}_i$ does not depend on $i$'s report, but only on the task allocation outcome $\pi^*(\theta)$.

In what follows, we show that for each agent $i$, if the others report types truthfully, then $i$'s utility is maximized by reporting her type truthfully.

Given an agent $i$' of type $\theta_i$ and the others' true type profile ${\theta}_{-i}$, assume that $i$ reported $\hat{\theta}_i \neq \theta_i$. For the allocation $\tau = \pi^*(\hat{\theta}_i, \theta_{-i})$, according to $x^{PEV}_i$, when $i$ finally completes her tasks, $i$'s utility is $u_i^1 = v_i(\tau, (1, p_{-i}^\tau)) - h_i({\theta}_{-i}) + V^1_{-i}((\hat{\theta}_i, \theta_{-i}), \pi^*)$ and her utility if she fails is $u_i^0 = v_i(\tau, (0, p_{-i}^\tau)) - h_i({\theta}_{-i}) + V^0_{-i}((\hat{\theta}_i, \theta_{-i}), \pi^*)$. Note that $i$'s expected valuation depends on her true valuation $v_i$ and all agents' true PoS. Therefore, $i$'s expected utility is:
\begin{align}
p_i^\tau\times& u_i^1 + (1-p_i^\tau)\times u_i^0 = \nonumber \\
&p_i^\tau\times v_i(\tau, (1, p_{-i}^\tau)) \label{eq_1} \\
&+ (1-p_i^\tau)\times v_i(\tau, (0, p_{-i}^\tau)) \label{eq_2} \\
&+ p_i^\tau\sum_{j \in N\setminus \{i\}} {v}_j(\tau, (1, p_{-i}^\tau)) \label{eq_3} \\
&+ (1-p_i^\tau)\sum_{j \in N\setminus \{i\}} {v}_j(\tau, (0, p_{-i}^\tau)) \label{eq_4} \\
&- h_i({\theta}_{-i}). \nonumber
\end{align} 
Since all valuations are multilinear in PoS, the sum of \eqref{eq_1} and \eqref{eq_2} is equal to $v_i(\tau, p^\tau)$, and the sum of \eqref{eq_3} and \eqref{eq_4} is $\sum_{j \in N\setminus \{i\}} {v}_j(\tau, p^\tau)$. Thus, the sum of \eqref{eq_1}, \eqref{eq_2}, \eqref{eq_3} and \eqref{eq_4} is the social welfare under allocation $\pi^*(\hat{\theta}_i, {\theta}_{-i})$. The social welfare is maximized when $i$ reports truthfully because $\pi^*$ maximizes social welfare (note that this is not the case when $\theta_{-i}$ is not truthfully reported).
Moreover, $h_i({\theta}_{-i})$ is independent of $i$'s report and is the maximum social welfare that the others can achieve without $i$. Therefore, by reporting $\theta_i$ truthfully, $i$'s utility is maximized. 
\end{proof}

Theorem~\ref{thm_exic_com} shows that multilinearity in PoS is sufficient to truthfully implement $(\pi^*, x^{PEV})$ in an ex-post equilibrium (ex-post truthful), but not in a dominant strategy (truthful). It has been shown in similar settings that ex-post truthfulness is the best we can achieve here~\cite{Porter_2008,Ramchurn_2009,Stein_2011,Conitzer_2014}.

\comment{
It is not hard to show an example where an agent is incentivized to misreport if some agents have misreported.
Take the example given in Table~\ref{tab_eg1} again and apply $(\pi^*, x^{PEV})$. If $\tau$ is chosen by the mechanism when both agents reported truthfully, i.e. $p_1^\tau(1+ p_2^\tau) - c > 0$, and there is no other allocation with positive social welfare, then agent $1$'s expected utility is $p_1^\tau(1+ p_2^\tau) - c$. Now assume that agent $2$ misreported a $\hat{p}_2^\tau < p_2^\tau$ such that $p_1^\tau(1+ \hat{p}_2^\tau) - c < 0$. Then the optimal allocation will be empty and gives both agents zero utility. In this case, agent $1$ might have a chance to correct agent $2$'s misreport by misreporting a $\hat{p}_1^\tau > p_1^\tau$ such that $\hat{p}_1^\tau(1+ \hat{p}_2^\tau) - c > 0$, which will bring $\tau$ back to the optimal allocation and gives agent $1$ a positive utility $p_1^\tau(1+ p_2^\tau) - c > 0$ as her utility does not depend on what PoS they have reported.
}

\subsection{Multilinearity in PoS is also Necessary}
\label{sect_linearity_cont}
\noindent In the above we showed that multilinearity in PoS is sufficient for $(\pi^*, x^{PEV})$ to be ex-post truthful. Here we show that the multilinearity is also necessary. 

\begin{theorem}
\label{thm_necessary_lic}
If $(\pi^*, x^{PEV})$ is ex-post truthful for all type profiles $\theta \in \Theta$, then for all $i\in N$, $v_i$ is multilinear in PoS.
\end{theorem}
\begin{proof}
By contradiction, assume that $v_i$ of agent of type $\theta_i$ is not multilinear in PoS, i.e., there exist a ${\theta}_{-i}$, an allocation $\tau \in T$, and a $j\in N$ (without loss of generality, assume that $j \neq i$) such that: 
\begin{equation}
\label{eq_assume}
\scalebox{0.85}{$
v_i(\tau, p^\tau) \neq {p}_j^\tau\times v_i(\tau, (1, p_{-j}^\tau)) + (1-{p}_j^\tau)\times v_i(\tau, (0, p_{-j}^\tau))
$}
\end{equation}
Under efficient allocation choice function $\pi^*$, it is not hard to find a type profile $\hat{\theta}_{-i}$ such that $\pi^*(\theta_i, \hat{\theta}_{-i}) = \tau$ and the PoS profile is the same between $\theta_{-i}$ and $\hat{\theta}_{-i}$. We can choose $\hat{\theta}_{-i}$ by setting $\hat{v}_j(\tau, p^\tau)$ to a sufficiently large value for each $j\neq i$.

Applying $(\pi^*, x^{PEV})$ on profile $(\theta_i, \hat{\theta}_{-i})$, when $j$ finally successfully completes her tasks $\tau_j$,
her utility is $u_j^1 = \hat{v}_j(\tau, (1, p_{-j}^\tau)) - h_j((\theta_i, \hat{\theta}_{-i})_{-j}) + V^1_{-j}((\theta_i, \hat{\theta}_{-i}), \pi^*)$ and her utility if she fails is $u_j^0 = \hat{v}_j(\tau, (0, p_{-j}^\tau)) - h_j((\theta_i, \hat{\theta}_{-i})_{-j}) + V^0_{-j}((\theta_i, \hat{\theta}_{-i}), \pi^*)$. Thus, $j$'s expected utility is (note that $\hat{p}_j^\tau = p_j^\tau$):
\begin{align}
p_j^\tau\times& u_j^1 + (1-p_j^\tau)\times u_j^0 = \nonumber \\
&p_j^\tau\times v_i(\tau, (1, p_{-j}^\tau)) \label{eq_1-1} \\
&+ (1-p_j^\tau)\times v_i(\tau, (0, p_{-j}^\tau)) \label{eq_2-1} \\
&+ p_j^\tau\sum_{k \in N\setminus \{i\}} \hat{v}_k(\tau, (1, p_{-j}^\tau)) \label{eq_3-1} \\
&+ (1-p_j^\tau)\sum_{k \in N\setminus \{i\}} \hat{v}_k(\tau, (0, p_{-j}^\tau)) \label{eq_4-1} \\
&- h_j({\theta}_{-j}). \nonumber
\end{align}
Given the assumption \eqref{eq_assume}, terms \eqref{eq_1-1} and \eqref{eq_2-1} together can be written as $v_i(\tau, p^\tau) + \delta_i$ where $\delta_i = \eqref{eq_1-1} + \eqref{eq_2-1} - v_i(\tau, p^\tau)$. Similar substitutions can be carried out for all other agents $k\in N\setminus \{i\}$ in terms \eqref{eq_3-1} and \eqref{eq_4-1} regardless of whether $v_k$ is mutlilinear in PoS. After this substitution, $j$'s utility can be written as:
\begin{align}
p_j\times& u_j^1 + (1-p_j)\times u_j^0 = \nonumber \\
&v_i(\tau, p^\tau) + \sum_{k \in N\setminus\{i\}} \hat{v}_k(\tau, p^\tau)  \label{eq_3-2} \\
&+ \sum_{k \in N} \delta_k \label{eq_4-2} \\
&- h_j(\theta_{-j}). \nonumber
\end{align}
Now consider a suboptimal allocation $\hat{\tau} \neq \tau$, if $\hat{\tau}$ is chosen by the mechanism, then $j$'s utility can be written as:
\begin{align}
\hat{u}_j =& \nonumber\\
&v_i(\hat{\tau}, p^{\hat{\tau}}) + \sum_{k \in N\setminus\{i\}} \hat{v}_k(\hat{\tau}, p^{\hat{\tau}})  \label{eq_3-3} \\
&+ \sum_{k \in N} \hat{\delta}_k \label{eq_4-3} \\
&- h_j(\theta_{-j}). \nonumber
\end{align}
In the above two utility representations, we know that terms $\eqref{eq_3-2} > \eqref{eq_3-3}$ 
because $\pi^*$ is efficient, but terms \eqref{eq_4-2} and \eqref{eq_4-3} 
can be any real numbers. 

\textit{In what follows, we tune the valuation of $j$ such that the optimal allocation is either $\tau$ or $\hat{\tau}$, and in either case $j$ is incentivized to misreport.}

In the extreme case where all agents except $i$'s valuations are multilinear in PoS, we have $\delta_{k} = 0, \hat{\delta}_{k} = 0$ for all $k \neq i$ in \eqref{eq_4-2} and \eqref{eq_4-3}. Therefore, $\sum_{k \in N} \delta_k = \delta_i \neq 0$ and $\sum_{k \in N} \hat{\delta}_k = \hat{\delta}_i$ (possibly $=0$). It might be the case that $\delta_i = \hat{\delta}_i$, but there must exist a setting where $\delta_i \neq \hat{\delta}_i$, otherwise $v_i$ is multilinear in PoS, because constant $\delta_i$ for any PoS does not violate the multilinearity definition.
\begin{enumerate}
\item If $\delta_i > \hat{\delta}_i$, we have $\eqref{eq_3-2} +\delta_i > \eqref{eq_3-3} +\hat{\delta}_i$. 
In this case, we can increase $\hat{v}_j(\hat{\tau}, p^{\hat{\tau}})$ such that $\hat{\tau}$ becomes optimal, i.e., $\eqref{eq_3-2} < \eqref{eq_3-3}$, but $\eqref{eq_3-2} +\delta_i > \eqref{eq_3-3} +\hat{\delta}_i$ still holds. 
Therefore, if $j$'s true valuation is the one that chooses $\hat{\tau}$ as the optimal allocation, then $j$ would misreport to get allocation $\tau$ which gives her a higher utility. 
\item If $\delta_i < \hat{\delta}_i$, we can easily modify $\hat{v}_j(\hat{\tau}, p^{\hat{\tau}})$ such that $\eqref{eq_3-2} +\delta_i < \eqref{eq_3-3} +\hat{\delta}_i$ but $\eqref{eq_3-2} > \eqref{eq_3-3}$ still holds. 
In this case, if $j$'s true valuation again is the one just modified, $j$ would misreport to get allocation $\hat{\tau}$ with a better utility.
\end{enumerate}
In both of the above situations, agent $j$ is incentivized to misreport, which contradicts that $(\pi^*, x^{PEV})$ is ex-post truthful. Thus, $v_i$ has to be multilinear in PoS.
\end{proof}

It is worth mentioning that Theorem~\ref{thm_necessary_lic} does not say that given a specific type profile $\theta$, all $v_i$ have to be multilinear in PoS for $(\pi^*, x^{PEV})$ to be ex-post truthful. 
Take the delivery example from Table~\ref{tab_eg1} and change agent $2$'s valuation for $\tau$ 
to be $v_2(\tau, p^{\tau}) = \mathbf{(p_1^{\tau})^2}\times p_2^{\tau}$ which is not multilinear in PoS.
It is easy to check that under this change, no agent can gain anything by misreporting if the other agent reports truthfully. However, given each agent $i$ of valuation $v_i$, to truthfully implement $(\pi^*, x^{PEV})$ in an ex-post equilibrium for all possible type profiles of the others, Theorem~\ref{thm_necessary_lic} says that $v_i$ has to be multilinear in PoS, otherwise, there exist settings where some agent is incentivized to misreport.

\subsection{Conditions for Achieving Individual Rationality}
\noindent
PEV-based mechanism is individually rational in Porter et al.~(\citeyear{Porter_2008})'s specific setting. However, in the general model we consider here, it may not guarantee this property. 
For example, there is an allocation where an agent has no task to complete in an allocation, but has a negative valuation for the completion of the tasks assigned to the others (i.e. she is penalised if the others complete their tasks). If that allocation is the optimal allocation and the allocation does not change with or without that agent, then she will get a zero payment therefore a negative utility.

Proposition~\ref{pro_ir} shows by restricting agents' valuations to some typical constraint, PEV-based mechanism can be made individually rational. 
The constraint says if an agent is not involved in a task allocation (i.e., when the tasks assigned to her is empty), she will not be penalised by the completion of the others' tasks.

\begin{proposition}
\label{pro_ir}
Mechanism $(\pi^*, x^{PEV})$ is individually rational if and only if for all $i\in N$, for all $\tau \in T$, if $\tau_i = \emptyset$, then $v_i(\tau, p^\tau) \geq 0$ for any $p^\tau \in [0,1]^N$.
\end{proposition}  
\begin{proof}
(If part) For all type profile $\theta \in \Theta$, for all $i\in N$, let $\tau = \pi^*(\theta)$ and $\hat{\tau} = \pi^*(\theta_{-i})$, $i$'s utility is given by $\sum_{k\in N} v_k(\tau, p^\tau) - \sum_{k\in N\setminus \{i\}} v_k(\hat{\tau}, p_{-i}^{\hat{\tau}})$, where the first term is the optimal social welfare with $i$'s participation and the second term is the optimal social welfare without $i$'s participation. It is clear that $\hat{\tau}_i = \emptyset$ as $\hat{\tau}$ is the optimal allocation without $i$'s participation. $\sum_{k\in N\setminus \{i\}} v_k(\hat{\tau}, p_{-i}^{\hat{\tau}}) + v_i(\hat{\tau}, p^{\hat{\tau}})$ is the social welfare for allocation $\hat{\tau}$. Since $\tau$ is optimal, we get that $\sum_{k\in N} v_k(\tau, p^\tau) \geq \sum_{k\in N\setminus \{i\}} v_k(\hat{\tau}, p_{-i}^{\hat{\tau}}) + v_i(\hat{\tau}, p^{\hat{\tau}})$. Thus, $\sum_{k\in N} v_k(\tau, p^\tau) - \sum_{k\in N\setminus \{i\}} v_k(\hat{\tau}, p_{-i}^{\hat{\tau}}) \geq v_i(\hat{\tau}, p^{\hat{\tau}}) \geq 0$, i.e. $i$'s utility is non-negative.

(Only if part) If there exist an $i$ of type $\theta_i$, a $\tau$, a $p^\tau \in [0,1]^N$ such that $\tau_i = \emptyset$ and $v_i(\tau, p^\tau) < 0$. We can always find a profile $\hat{\theta}_{-i}$ s.t. $\hat{p}^\tau = p^\tau$ and $\pi^*(\theta_i, \hat{\theta}_{-i}) = \pi^*(\hat{\theta}_{-i}) = \tau$. It is clear that the payment for $i$ is $0$ and her utility is $v_i(\tau, p^\tau) < 0$ (violates individual rationality).
\end{proof}

\section{Extension to Trust-Based Environments}
\label{sect_trust}
So far, we have assumed that each agent can correctly predict her probability of success (PoS) for each task, but in some environments, an agent's PoS is not perfectly perceived by the agent alone. Instead, multiple other agents may have had prior experiences with a given agent and their experiences can be aggregated to create a more informed measure of the PoS for the given agent. This measure is termed the trust in the agent~\cite{Ramchurn_2009}.
Ramchurn et al. have extended Porter et al.'s mechanism to consider agents' trust information and showed that the extension is still truthfully implementable in their settings.

Similarly, our general model can also be extended to handle the trust information by changing singleton $p_i^\tau$ to be a vector $p_i^\tau = (p_{i,1}^\tau, ..., p_{i,j}^\tau, ..., p_{i,n}^\tau)$ where $p_{i,j}^\tau$ is the probability that $i$ believes $j$ will complete $j$'s tasks in $\tau$. Agent $i$'s aggregated/true PoS for task $\tau$ is given by a function $f_i^\tau: [0,1]^N \rightarrow [0,1]$ with input $(p_{1,i}^\tau,..., p_{n,i}^\tau)$. Given this extension, for any type profile $\theta$, let $\rho_i^\tau = f_i^\tau(p_{1,i}^\tau,..., p_{n,i}^\tau)$, the social welfare of a task allocation $\tau$ is defined as: 
\begin{equation}
\label{eq_trusteff}
\sum_{i\in N} v_i(\tau, \rho^\tau)
\end{equation}
where $\rho^\tau = (\rho_1^\tau, ..., \rho_n^\tau)$.

As shown in \cite{Ramchurn_2009}, PEV-based mechanism can be extended to handle this trust information by simply updating the efficient allocation choice function $\pi^*$ with the social welfare calculation given by Equation \eqref{eq_trusteff}. Let us call the extended mechanism $\mathcal{M}^{trust}$.
Ramchurn et al. have demonstrated that $\mathcal{M}^{trust}$ is ex-post truthful in their settings when the PoS aggregation function is the following linear form:
\begin{equation}
\label{eq_linear_aggregation}
f_i^\tau(p_{1,i}^\tau,..., p_{n,i}^\tau) = \displaystyle \sum_{j\in N} \omega_j\times p_{j,i}^\tau
\end{equation}
where constant $\omega_j\in [0,1]$ and $\sum_{j\in N} \omega_j = 1$.

Following the results in Theorems~\ref{thm_exic_com} and \ref{thm_necessary_lic}, we generalize Ramchurn et al.'s results to characterize all aggregation forms under which $\mathcal{M}^{trust}$ is ex-post truthful. 

\begin{definition}
\label{def_linear_trust}
A PoS aggregation $f_i = (f_i^\tau)_{\tau\in T}$ is \textbf{multilinear} if for all $j \in N$, for all $\tau \in T$, for all $\theta \in \Theta$, $f_i^\tau(p_{1,i}^\tau,...,p_{j,i}^\tau,...,p_{n,i}^\tau) = p_{j,i}^\tau \times f_i^\tau(p_{1,i}^\tau,...,p_{j-1,i}^\tau,1, p_{j+1,i}^\tau,...,p_{n,i}^\tau) + (1-p_{j,i}^\tau) \times f_i^\tau(p_{1,i}^\tau,...,p_{j-1,i}^\tau,0, p_{j+1,i}^\tau,...,p_{n,i}^\tau)$.
\end{definition}

Definition \ref{def_linear_trust} is similar to the multilinear in PoS definition given by Definition \ref{def_linear}. Multilinear aggregations cover the linear form given by Equation \eqref{eq_linear_aggregation}, but also consist of many non-linear forms such as $\prod_{j\in N} p_{j,i}^\tau$. 
The following corollary directly follows Theorems~\ref{thm_exic_com} and \ref{thm_necessary_lic}. We omit the proof here. The basic idea of the proof is that given a multilinear function, if we substitute another multilinear function (with no shared variables) for one variable of the function, then the new function must be multilinear.

\begin{corollary}
Trust-based mechanism $\mathcal{M}^{trust}$ is ex-post truthful if and only if for all $i\in N$, $v_i$ is multilinear in PoS, and the PoS aggregation $f_i$ is multilinear.
\end{corollary} 

For $\mathcal{M}^{trust}$ to be individually rational, the constraint specified in Proposition~\ref{pro_ir} is still sufficient and necessary, if we change $h_{-i}$ in the payment definition (Equation~\eqref{eq_commit_pay}) to be the optimal social welfare that the others can achieve without $i$, but assume that $i$ offered the worst trust in the others (see \cite{Ramchurn_2009} for more details).

\section{Discussions}
\subsection{Link to General Interdependent Valuations}
\label{sect_link}
So far, we have characterised the applicability of PEV-based mechanism and its extension with trust in a general task allocation setting. We should also note that there exists a body of research for general interdependent valuations such as~\cite{Milgrom_1982,Jehiel_2001}. Hence, in what follows we draw the parallels between the two areas and compare and contrast their key results and assumptions.


The work of \cite{Jehiel_2001} is especially interesting to this study, because they have identified a necessary condition for implementing an efficient and Bayes-Nash truthful\footnote{\emph{Bayes-Nash truthful} is weaker than ex-post truthful and it assumes that all agents know the correct probabilistic distribution of each agent's type.} mechanism (see Theorem 4.3 in \cite{Jehiel_2001}). However, their setting and the necessary condition do not apply to our setting, because:
\begin{enumerate}
\item The model in \cite{Jehiel_2001} can only model one special setting of our problem, namely the setting where the tasks between agents are independent. Also it is impossible to model trust at the same time.
\item The mechanism considered in \cite{Jehiel_2001} has no ability to verify agents' reports. 
\end{enumerate}
Therefore, we can see that our problem is a very special interdependent valuation setting, which allows the mechanism to partially verify agents' reports and to design mechanisms with better performance.

\subsection{When Agents are Not Risk-Neutral}
We have shown that as soon as agents are risk-neutral with respect to their execution uncertainty, PEV-based mechanism is sufficient to provide incentives for agents to reveal their true types. However, in many real-world applications, participants are often not risk-neutral. For instance, when we reserve a ride from a taxi/carsharing company to catch a flight, we certainly do not want to take risk to get an unreliable booking. On the other hand, we often face challenging tasks that are very unlikely to be successfully completed (for example open research questions and financial investments), but we are very willing to take risks to try. Our results indicate that, to handle these non-risk-neutral settings, we need better solutions. 

Furthermore, when agents are not risk-neutral, individual rationality (Definition~\ref{def_ir}) needs to be redefined, as the current definition assumes that agents are risk-neutral with respect to their execution uncertainty.

\subsection{Challenge of the Efficient Allocation Design}
In our model, we assumed that the set of possible task allocation outcomes are given and the efficient task allocation is chosen from that set. It is worth mentioning that given a specific task allocation setting, finding an efficient allocation may not come so easy, e.g.,~\cite{Ramchurn_2009,Stein_2011,Feige_2011,Conitzer_2014}. If it is computationally hard to get an efficient outcome, there exist techniques to tackle it without violating the truthfulness properties, e.g.,~\cite{Nisan_2007}. 

\section{Conclusions}
\label{sect_con}
We studied a general task allocation problem where multiple agents collaboratively accomplish a set of tasks, but they may fail to successfully complete tasks assigned to them. To design an efficient task allocation mechanism for this problem, we showed that post-execution verification based mechanism is truthfully implementable, if and only if all agents are risk-neutral with respect to their execution uncertainty. We also showed that trust information between agents can be integrated into the mechanism without violating its properties, if and only if the trust information is aggregated by a multilinear function. This characterisation will help us further study specific task allocation settings. As mentioned in the above discussions, one very interesting future work is to design efficient mechanisms for task allocation settings with non-risk-neutral participants.

\bibliographystyle{aaai}
\bibliography{aaai16rs}

\end{document}